\documentclass[twoside]{article}

\usepackage[accepted]{aistats2025}
\usepackage{amsmath}
\usepackage{amssymb} 
\usepackage{xcolor} 
\usepackage{algorithm,algpseudocode}
\usepackage{amsthm}  %
\usepackage{caption}
\usepackage{subcaption}
\usepackage{amsmath}
\usepackage{amssymb}

\usepackage{authblk}
\usepackage[utf8]{inputenc} 
\usepackage[T1]{fontenc}    
\usepackage{hyperref}       
\usepackage{url}            
\usepackage{booktabs}       
\usepackage{amsfonts}       
\usepackage{nicefrac}       
\usepackage{microtype}      
\usepackage{xcolor}         
\usepackage{amsthm} 
\usepackage{amssymb}  
\usepackage{breqn}
\usepackage{array}

\newcommand{\Var}{\text{Var}}

\newcommand{\vP}{\mathbf{P}}
\newcommand{\vQ}{\mathbf{Q}}

\newcommand{\vdx}{\mathbf{dX}}
\newcommand{\vX}{\mathbf{X}}
\newcommand{\vT}{\mathbf{T}}
\newcommand{\vY}{\mathbf{Y}}
\newtheorem{lemma}{Lemma}
\newtheorem{theorem}{Theorem}

\newtheorem{remark}{Remark}

\DeclareMathOperator{\E}{\mathbb{E}}
\begin{document}

%

%

\twocolumn[

\aistatstitle{Asymptotically  Optimal Change Detection for Unnormalized  Pre- and Post-Change Distributions }

\aistatsauthor{Arman Adibi\textsuperscript{1} \And Sanjeev Kulkarni\textsuperscript{1} \And H. Vincent Poor\textsuperscript{1} \And Taposh Banerjee\textsuperscript{2} \And Vahid Tarokh\textsuperscript{3}}

\aistatsaddress{\textsuperscript{1}Princeton University \And \textsuperscript{2}University of Pittsburgh \And \textsuperscript{3}Duke University}
 
 ]

\begin{abstract}
This paper addresses the problem of detecting changes when only unnormalized pre- and post-change distributions are accessible. This situation happens in many scenarios in physics such as in ferromagnetism, crystallography, magneto-hydrodynamics, and thermodynamics, where the energy models are difficult to normalize.

Our approach is based on the estimation of the Cumulative Sum (CUSUM)  statistics, which is known to produce optimal performance. We first present an intuitively appealing approximation method. Unfortunately, this produces a biased estimator of the CUSUM statistics and may cause performance degradation. We then propose the Log-Partition Approximation Cumulative Sum (LPA-CUSUM) algorithm based on thermodynamic integration (TI)  in order to estimate  the log-ratio of normalizing constants of pre- and post-change distributions. It is proved that this approach gives an unbiased estimate of the log-partition function and the CUSUM statistics, and leads to an asymptotically optimal performance. Moreover, we derive a relationship between the required sample size for thermodynamic integration and the desired detection delay performance, offering guidelines for practical parameter selection. Numerical studies are provided demonstrating the efficacy of our approach.
\end{abstract}
\section{INTRODUCTION}
Quickest change detection is a fundamental problem in various domains, such as control, signal processing, machine learning, and finance \cite{veeravalli2001decentralized,xie2021sequential}. In many real-world applications in physics and machine learning \cite{holzmuller2023convergence}, we have access to the unnormalized version of the distributions, while most change detection algorithms require access to the normalized version of the distributions, necessitating computation of the normalizing constants. However, computing the normalization constants of distributions is challenging because it often involves evaluating integrals or summations over high-dimensional spaces, which is typically infeasible \cite{huber2015approximation}. Specifically, for many energy-based models, this computation is NP-hard and even when feasible, it remains computationally difficult and expensive \cite{kolmogorov2018faster}. This paper tackles the problem of change detection when the normalizing constants of the pre-change and post-change distributions are intractable.

We propose a change detection method that utilizes thermodynamic integration (TI), a technique from statistical physics that estimates intractable normalizing constants for high-dimensional distributions. TI relies on the insight that estimating the ratio of two unknown normalizing constants is more feasible than directly computing the constants themselves.

Our approach builds upon the CUSUM algorithm proposed by  \cite{page1955test}, which is designed for sequential change detection when the pre-change and post-change distributions are known. However, in our setting, we do not have access to the normalizing constants of these distributions. By incorporating TI into the CUSUM algorithm, we can estimate the log-ratio of the normalizing constants using an unbiased estimator with bounded variance.

The main contributions of this paper are:

\begin{itemize}
   \item We introduce a novel change detection algorithm, LPA-CUSUM, that combines the CUSUM framework with TI to handle scenarios where the normalizing constants of the pre-change and post-change distributions are intractable. 
   \item We provide a theoretical analysis of the false alarm and delay properties of the proposed LPA-CUSUM algorithm, demonstrating its statistical guarantees. Furthermore, we prove that LPA-CUSUM is asymptotically optimal.
   \item We derive a relationship between the number of samples required for the TI estimator and the desired delay performance, providing practical guidelines for setting the algorithm's parameters.
   \item We evaluate the performance of LPA-CUSUM on synthetic datasets, showcasing its effectiveness in detecting changes in complex distributions.
\end{itemize}

\section{RELATED WORKS}
\subsection{Change Detection}
Change detection algorithms play a critical role in various fields, including sensor networks, cyber-physical systems, and neuroscience, where detecting abrupt changes in data streams is important  \cite{veeravalli2014quickest,poor2008quickest}. In the literature, various optimality results have been established for classical change detection algorithms  \cite{shiryaev1963optimum,moustakides1986optimal,page1955test,roberts1966comparison,lai1998information,tartakovsky2014sequential}. Notably, the Cumulative Sum (CUSUM) algorithm has been shown to be asymptotically optimal under certain formulations, such as Lorden's problem and Pollak's problem  \cite{lorden1971procedures,pollak1985optimal}.

However, these optimality results are based on the assumption of known distributions before and after the change, which may not hold in practical scenarios \cite{yu2016statistical,chen2015graph,nalisnick2018deep,wu2023quickest}. Moreover, likelihood-based algorithms, such as the CUSUM algorithm, are widely used but may face difficulties in scenarios where explicit distributions are computationally challenging to compute. To address these challenges, Wu et al.~\cite{wu2023quickest} proposed the Score-based CUSUM (SCUSUM) algorithm, which enables change detection for unnormalized statistical models with unknown normalization constants or when the Hyv\"arinen Score of the distributions are available. While SCUSUM
is efficient when exactly computing the normalizing
constants $Z_1$, and $Z_0$ is not possible, in many cases, we
can compute an estimator of the normalizing
constants, and we will show how the use of such an
estimator can help us achieve performance close to the
CUSUM algorithm.

\subsection{Partition Function Estimation}\label{}
Estimating partition functions is a key challenge for complex distributions such as energy-based models (Gibbs distributions). One of the most common approaches is annealed importance sampling (AIS), introduced by Neal~ \cite{neal2001annealed}. AIS has become a standard method for partition function estimation, owing to its ability to handle complex probabilistic models. Subsequent works have addressed limitations of the original AIS algorithm by improving convergence properties and extending theoretical understanding, such as in the works by Karagiannis and Andrieu~ \cite{karagiannis2013annealed} and Holzmüller et al.~ \cite{holzmuller2023convergence}. Adaptive versions of AIS have been proposed to enhance efficiency, such as the adaptive AIS by Štefankovič et al.~ \cite{vstefankovivc2009adaptive}, which dynamically adjusts the annealing schedule based on observed performance, thereby improving computational efficiency in challenging scenarios. For discrete distributions, Haddadan et al.~ \cite{haddadan2021fast} developed the fast annealed importance sampling (FAIS) algorithm, which combines AIS with sequential Monte Carlo methods to achieve significant computational savings while retaining accuracy.

Another widely used technique for log-partition function estimation is thermodynamic integration (TI), which has applications in both physics and chemistry  \cite{frenkel2023understanding,friel2012estimating}. TI is based on calculating the ratio of unknown normalizing constants rather than the constants themselves, thereby simplifying the estimation process  \cite{kirkwood1935statistical,gelman1998simulating}. TI provides an unbiased estimator with bounded variance, making it suitable for complex, high-dimensional distributions. Ge et al.~ \cite{ge2020estimating} analyzed an annealing algorithm combined with multilevel Monte Carlo sampling for estimating log-partition functions in log-concave settings, providing important theoretical insights, including information-based lower bounds on achievable convergence rates. Recent work by Marteau-Ferey et al.~ \cite{marteau2022sampling} proposed an approach involving sampling via log-partition function estimation, which can be particularly useful in high-dimensional scenarios. Bach~ \cite{bach2022information, bach2024sum} further explored this domain by proposing sum-of-squares relaxations for variational inference, offering a new direction for addressing the challenges of partition function estimation.\vspace{-1mm}
\section{PROBLEM FORMULATION}
Quickest change detection refers to the problem of determining whether or not a change has occurred in a sequence of independent random variables $\{\vX_n\}_{n \geq 1}$, defined on the probability space $(\Omega, \mathcal{F}, \vQ_\nu)$. For each $n$, let $\mathcal{F}_n$ be the $\sigma$-algebra generated by the random variables $\vX_1, \vX_2, \dots, \vX_n$. We define $F$ to be the $\sigma$-algebra generated by the union of all sub-$\sigma$-algebras $\mathcal{F}_n$, i.e., $\mathcal{F}=\sigma(\cup_{n\geq1} \mathcal{F}_n)$.

Under $\vQ_\nu$, $\vX_1, \vX_2, \dots, \vX_{\nu-1}$ are i.i.d. according to a density $\vP_0$, and $\vX_\nu, \vX_{\nu+1}, \dots$ are i.i.d. according to a density $\vP_1$. We think of $\nu$ as the change point, $\vP_0$ as the pre-change density, and $\vP_1$ as the post-change density. We use $\E_\nu$ and $\text{Var}_\nu$ to denote the expectation and the variance associated with the measure $\vQ_\nu$, respectively. We use $\vP_0$ to denote the measure under which there is no change, with $\E_0$ denoting the corresponding expectation.

A change detection algorithm is a stopping time $T$ with respect to the data stream $\{\vX_n\}_{n\geq1}$:
\begin{align*}
\{T \leq n\} \in \mathcal{F}_n.
\end{align*}

If $T \geq \nu$, we have made a delayed detection; otherwise, a false alarm has happened. Intuitively, there is a trade-off between detection delay and false alarms. We consider two minimax problem formulations to find the best stopping rule.

In  \cite{lorden1971procedures}, the following metric, the worst-case averaged detection delay (WADD), is defined:
\begin{align*}
L_{\text{WADD}}(T) = \sup_{\nu\geq1} \text{ess sup} \, \E_\nu[(T - \nu + 1)^+ | \mathcal{F}_{\nu-1}], 
\end{align*}
where $(y)^+ = \max(y, 0)$ for any $y \in \mathbb{R}$. This leads to the minimax optimization problem
\begin{align}\label{eq:lordan}
\min_T L_{\text{WADD}}(T) \quad \text{subject to} \quad \E_0[T] \geq h.
\end{align}

We are also interested in the version of minimax metric introduced in Pollak  \cite{pollak1985optimal}, the worst conditional averaged detection delay (CADD):
\begin{align*}
L_{\text{CADD}}(T) = \sup_{\nu\geq1} \E_\nu[T - \nu | T \geq \nu].
\end{align*}

The optimization problem becomes 
\begin{align*}
\min_T L_{\text{CADD}}(T) \quad \text{subject to} \quad \E_0[T] \geq h.
\end{align*}

\subsection{The Likelihood Ratio-based CUSUM Algorithm}

Given the data stream $\{\vX_n\}_{n\geq1}$, the stopping rule of the likelihood ratio-based CUSUM algorithm is defined by
\begin{equation*}
\resizebox{1\hsize}{!}{$
\begin{aligned}
T_{\text{CUSUM}} = \inf\left\{ n \geq 1 : \max_{1\leq k \leq n} \sum_{i=k}^{n} \log \frac{\vP_1(\vX_i)}{\vP_0(\vX_i)} \geq \log h \right\},
\end{aligned}$}
\end{equation*}
where the infimum of the empty set is defined to be $+\infty$, and $\log h > 0$ is referred to as the stopping threshold. The value of this threshold is clearly related to the trade-off between detection delay and false alarms. It is known  \cite{page1955test} that $T_{\text{CUSUM}}$ can be written as
\begin{align*}
T_{\text{CUSUM}} = \inf\{ n \geq 1 : \Lambda(n) \geq \log h \},
\end{align*}
where $\Lambda(n)$ is defined using the recursion
\begin{align*}
\Lambda(0) = 0, \quad \Lambda(n) = \left(\Lambda(n - 1) + \log \frac{\vP_1(\vX_n)}{\vP_0(\vX_n)} \right)^{+}.
\end{align*}
In \cite{moustakides1986optimal}, it is demonstrated that the CUSUM algorithm achieves exact optimality for Lorden’s problem, under every fixed constraint $h$ (provided the threshold in the CUSUM algorithm is chosen so that the false alarm constraint is met with equality). Additionally, as noted in  \cite{lai1998information}, the algorithm is asymptotically optimal for Pollak’s problem. The asymptotic performance of the CUSUM algorithm is further examined in  \cite{lorden1971procedures} and  \cite{lai1998information}. Specifically, it is established that
\begin{equation}\label{eq:perf}
 \resizebox{1\hsize}{!}{$
\begin{aligned}
L_{\text{WADD}}(T_{\text{CUSUM}}) \sim L_{\text{CADD}}(T_{\text{CUSUM}}) \sim \frac{\log h}  {\mathbb{D}_{\text{KL}}(\vP_1 , \vP_0)}, 
\end{aligned}
$}
\nonumber
\end{equation}

 as $\gamma \rightarrow \infty$. Here, $\mathbb{D}_{\text{KL}}(\vP_1 , \vP_0)$ denotes the Kullback-Leibler divergence between the post-change distribution $\vP_1$ and the pre-change distribution $\vP_0$:
\begin{align*}
\mathbb{D}_{\text{KL}}(\vP_1 , \vP_0)= \int_\vX \vP_1(\vX) \log \frac{\vP_1(\vX)}{\vP_0(\vX)} \, \vdx.
\end{align*}
Furthermore, the notation $g(c) \sim h(c)$ as $c \rightarrow c_0$ signifies that $\frac{g(c)}{h(c)} \rightarrow 1$ as $c \rightarrow c_0$ for any two functions $c \rightarrow g(c)$ and $c \rightarrow h(c)$.

\subsection{Unnormalized Statistical Models}
In applications where obtaining access to $\vP_1, \vP_0$ is not possible, the CUSUM algorithm cannot be used. Specifically, the focus is on the case where we only have access to the unnormalized versions of the distributions, $\tilde{\vP}_1(\vX)$ and $\tilde{\vP}_0(\vX)$, where $\vP_1(\vX) = \frac{\tilde{\vP}_1(\vX)}{Z_1}$ and $\vP_0(\vX) = \frac{\tilde{\vP}_0(\vX)}{Z_0}$.

An alternative method proposed by Wu et al.~\cite{wu2023quickest} is the Score-based CUSUM (SCUSUM). SCUSUM is a novel variant of the CUSUM algorithm that replaces the log-likelihood with the Hyv\"arinen score and provides a delay guarantee in terms of the Fisher divergence. While SCUSUM is efficient when computing the normalizing constants $Z_1, Z_0$ is not possible, in many cases, we can calculate an estimator of the normalizing constants, and we will show how the use of such an estimator can help us achieve performance close to the CUSUM algorithm.

\section{PROPOSED ALGORITHM}
\begin{algorithm}[t]

\caption{LPA-CUSUM Algorithm}
\label{alg:LPA_CUSUM}
\begin{algorithmic}[1]
\State Initialize $Z(0) \gets 0$
\For{$n = 1, 2, \dots$}
   \State Compute $\log\frac{\tilde{\vP}_1(\vX_n)}{\tilde{\vP}_0(\vX_n)}$
   \State Obtain $\vY_{1,n}, \vY_{2,n}, \dots, \vY_{i,n}$ from oracle $\mathcal{A}$
   \State Compute ${\vT}_{i,n} = \frac{\vY_{1,n} + \vY_{2,n} + \dots + \vY_{i,n}}{i}$
   \State Update $$Z(n) = \left(Z(n - 1) + \gamma\log\frac{\tilde{\vP}_1(\vX_n)}{\tilde{\vP}_0(\vX_n)} + \gamma{\vT}_{i,n}\right)^+$$
   \If{$Z(n) \geq \log h$}

   \State  Let $\tau= n$.
       \State Stop and declare change.
       
   \EndIf
\EndFor
\end{algorithmic}
\end{algorithm}

Now, we will consider the case where we have two unnormalized distributions $\tilde{\vP}_1(\vX)$ and $\tilde{\vP}_0(\vX)$ where $\vP_1(\vX) = \frac{\tilde{\vP}_1(\vX)}{Z_1}$ and $\vP_0(\vX) = \frac{\tilde{\vP}_0(\vX)}{Z_0}$, and we don't have access to $Z_1$ and $Z_0$. But, we have an oracle $\mathcal{A}$ (see section \ref{subsec:TI} for more details about oracle $\mathcal{A}$) that can generate an unbiased estimator of $\log\frac{Z_0}{Z_1}$ with variance $\sigma^2$.
 We want to show that using this oracle, we can run the Algorithm \ref{alg:LPA_CUSUM} that approximately behaves similar to CUSUM. Specifically, we want to use the Algorithm \ref{alg:LPA_CUSUM}
which uses ${\vT}_{i,n}$, defined as
\begin{align}
{\vT}_{i,n} := \frac{\vY_{1,n} + \vY_{2,n} + \dots + \vY_{i,n}}{i},
\end{align}
Where $\vY_{k,n}$ are independent and identically distributed outputs of the oracle $\mathcal{A}$ for all $k\geq 0$, and $n\geq 1$. These outputs are unbiased estimators of $\log \frac{Z_0}{Z_1}$ with variance $\sigma^2$. $\vY_{k,n}$ are independent of $\vX_j$ for all $k,j\geq 0$, and $n\geq 1$. The update rule in Algorithm \ref{alg:LPA_CUSUM} is based on the variable $ Z(n) $, which is defined recursively as
\begin{align*}
Z(n) := \left(Z(n - 1) + \gamma \log \frac{\tilde{\vP}_1(\vX_n)}{\tilde{\vP}_0(\vX_n)} + \gamma \vT_{i,n}\right)^+.
\end{align*}
We also define the stopping point of the algorithm as
\begin{align*}
\tau := \inf \{n \geq 1 : Z(n) \geq \log h\}.
\end{align*}

\section{THEORETICAL  ANALYSIS}
In this section, we analyze the performance of the delay and false alarm rate for the algorithm \ref{alg:LPA_CUSUM}. 

\subsection{Controlling the False Alarm}
In this section, we will focus on analyzing the false alarm rate. We will begin by presenting the preliminary lemmas that will be used in the proof of the Theorem \ref{thm:1}. Lemma \ref{lemma:supermartingle} establishes an important inequality concerning the maximum value attained by a supermartingale process. A supermartingale is a stochastic process that, on average, decreases over time, $\E[\vX_{n+1}\mid \vX_1,\ldots,\vX_n] \leq \vX_n$. In this context, $\vX_k$ represents such a process at time $k$. The lemma bounds the probability that the maximum value of this process up to time $n$ exceeds a certain threshold $c$.

\begin{lemma}[ \cite{chow1971great}]\label{lemma:supermartingle}
Let $\{\vX_k\}_{k \geq 1}$ be a non-negative supermartingale process defined on the probability space $(\Omega, \mathcal{F}, \mathbb P)$. Then, we have
$$
\mathbb{P}\left(\max_{1 \leq k \leq n} \vX_k \geq c\right) \leq \frac{1}{c} \E[\vX_1],
$$
for all $n \geq 1$ and any real constant $c > 0$.
\end{lemma}

Let us define $\vY$ as a random variable representing the unbiased estimators of $\log \frac{Z_0}{Z_1}$ with variance $\sigma^2$, obtained from oracle $\mathcal{A}$. The following lemmas hold.

\begin{lemma}\label{lemma:tn_y} For every $\gamma > 0$, Jensen's Inequality yields the following inequality for random variables ${\vT}_{i,n}$ and $\vY$:
    $$\E_{0}[\exp(\gamma {\vT}_{i,n})]\leq \E_{0}[\exp(\gamma \vY)] .$$

    \end{lemma}

The lemma above establishes an important relationship between the expectations of exponential functions of ${\vT}_{i,n}$ and $\vY$. Essentially, it shows that the expected value of the exponential of ${\vT}_{i,n}$ is bounded above by the exponential of the expected value of $\vY$, for any positive $\gamma$.

\begin{lemma}
\label{lem:existanceof good gamma}There exists $\gamma > 0$ such that the following inequality holds for random variable $\vY$ and for all $n\geq 1$:
    \begin{align}\label{eq:samegamma1}
    \E_{0}[\exp(\gamma\log\frac{\tilde \vP_1(\vX_n)}{\tilde \vP_0(\vX_n)}+\gamma \vY)] \leq 1.\end{align}
    
    Then, for the same $\gamma$, we have for all $n\geq 1$:
    \begin{align}
      \E_{0}[\exp(\gamma\log\frac{\tilde \vP_1(\vX_n)}{\tilde \vP_0(\vX_n)}+\gamma{\vT}_{i,n})] \leq 1.
    \end{align}

\end{lemma}     

The lemma implies that in order to find a suitable $\gamma$ that satisfies $\E_{0}[\exp(\gamma\log\frac{\tilde \vP_1(\vX)}{\tilde \vP_0(\vX)}+\gamma {\vT}_{i,n})] \leq 1,$ we can only select a $\gamma$ such that  $\E_{0}[\exp(\gamma\log\frac{\tilde \vP_1(\vX)}{\tilde \vP_0(\vX)}+\gamma \vY)] \leq 1.$

The following theorem proves a crucial property concerning the expected false alarm for the Algorithm \ref{alg:LPA_CUSUM}.

\begin{theorem}\label{thm:1}
Suppose we select a parameter $\gamma$ such that for all $n\geq 1$, \begin{align}\label{eq:goodgamma}
   \E_{0}[\exp(\gamma\log\frac{\tilde \vP_1(\vX_n)}{\tilde \vP_0(\vX_n)}+\gamma{\vT}_{i,n})] \leq 1. 
\end{align} Then, for all thresholds $h > 0$, we have $\E_{0}[\tau] \geq e^h$.
\end{theorem}
$\E_{0}\left[\tau\right]$ is referred to as the Average Run Length (ARL). As the stopping threshold increases ARL increases
at least exponentially. The proof of this Theorem can be found in the supplementary materials. 

\subsection{Delay Analysis}

In this section, we present the analysis of detection delays for the Algorithm \ref{alg:LPA_CUSUM}.

\begin{theorem}\label{thm:delay}
    For the delay of Algorithm \ref{alg:LPA_CUSUM}, we have that,
    
\begin{align}
L_{\text{WADD}}(\tau) \sim L_{\text{CADD}}(\tau) \sim \frac{\log h}  {\gamma\mathbb{D}_{\text{KL}}(\vP_1 , \vP_0)}. 
\end{align} 
\end{theorem}
Theorem~\ref{thm:delay} states that the delay performance of the Algorithm \ref{alg:LPA_CUSUM} is close to CUSUM up to a factor $\gamma$. The next section will show how to choose $\gamma$ properly. The proof of this Theorem is in the supplementary materials. 
\subsection{Number of Samples for LPA Estimator}
In this section, we drive an identity about choosing the proper $\gamma$.
\begin{theorem}\label{thm:asy_opt}
If we let $\gamma = \gamma_0=1 - \frac{\sigma^2+2\epsilon}{2i \mathbb{D}_{\text{KL}}(\vP_1, \vP_0)} $, we know that for any $\epsilon>0$, for a large enough $i$, we have
\begin{align}
\E_0\left[\exp\left(\gamma_0\log\frac{\vP_1(\vX)}{\vP_0(\vX)} + \gamma_0{\vT}_{i,n} -
\gamma_0\log\frac{Z_0}{Z_1}\right)\right] \leq 1,
\end{align}
and the expected delay for $\gamma_0$ is
\begin{align}\label{eq:choiceofgamma2}
L_{\text{CADD}} \sim \frac{\log h}{\mathbb{D}_{\text{KL}}(\vP_1, \vP_0) - \frac{\sigma^2+2\epsilon}{2i}}.
\end{align}
\end{theorem}
\begin{remark}
    Equation \eqref{eq:choiceofgamma2} provides an insight into the selection of $\gamma$. As the variance of the estimator, $\sigma^2$, decreases, we should choose a larger value of $\gamma$, which results in a lower detection delay. Similarly, if $\mathbb{D}_{\text{KL}}(\vP_1, \vP_0)$ is large, we should opt for a larger value of $\gamma$, which also leads to lower detection delay. Additionally, if the number of samples $i$ increases, we should select a larger value of $\gamma$, which results in a lower detection delay. When $\gamma$ approaches 1 (larger $i$, or smaller $\sigma^2$), Algorithm \ref{alg:LPA_CUSUM} behaves more like CUSUM since the estimator for log portion function becomes more precise. This demonstrates that, with an appropriate choice of $\gamma$, the Algorithm \ref{alg:LPA_CUSUM} asymptotically achieves performance similar to CUSUM in large sample regimes for log partition function estimation. Therefore, Algorithm \ref{alg:LPA_CUSUM} is \textbf{asymptotically optimal}.
\end{remark}
\subsection{Partition Function Estimation}

This section will discuss how to build oracle $\mathcal{A}$ in the Algorithm \ref{alg:LPA_CUSUM}. One approach to build an estimator of the log partition function is leveraging Thermodynamic Integration (TI) to estimate the log-ratio of the partition functions.

Suppose we have two unnormalized distributions $\tilde{\vP}_1(\vX)$ and $\tilde{\vP}_0(\vX)$, where $\vP_1(\vX) = \frac{\tilde{\vP}_1(\vX)}{Z_1}$ and $\vP_0(\vX) = \frac{\tilde{\vP}_0(\vX)}{Z_0}$ are the corresponding normalized distributions, with $Z_1$ and $Z_0$ being the partition functions or normalizing constants. The goal is to estimate the log-ratio $\log(Z_0/Z_1)$ using TI, as direct computation of the partition functions may be intractable for complex, high-dimensional distributions.
\subsubsection{Naive Approaches}
In this section, we discuss two methods for estimating the partition function and demonstrate that these approaches yield biased estimators, making them unsuitable for our setting. This motivates the use of Thermodynamic Integration, which offers a more promising alternative.

The first approach to estimate $Z_1$ and $Z_0$ involves approximating the integral definition:
\begin{align*}
    Z_i = \int \tilde{\vP}_i(\vX) \, d\vX, \quad i \in \{0, 1\}.
\end{align*}
If we denote these approximations by $\tilde{Z}_0$ and $\tilde{Z}_1$, we can use $\log\left(\frac{\tilde{Z}_0}{\tilde{Z}_1}\right)$ as an estimator for $\log\left(\frac{Z_0}{Z_1}\right)$. However, this estimator is biased, and using a biased estimator in Algorithm \ref{alg:LPA_CUSUM} will result in a high false alarm rate and poor detection delay, unless the bias is smaller than the change point, specifically $\mathcal{O}\left(\frac{1}{\nu}\right)$, which requires a large number of samples.

The second method to estimate the normalizing constant uses the identity
\begin{align}\label{eq:naive2}
    \frac{Z_0}{Z_1} = \mathbb{E}_{\vX \sim \vP_1}\left[\frac{\tilde{\vP}_0(\vX)}{\tilde{\vP}_1(\vX)}\right],
\end{align}
and employs the Monte Carlo estimator
\begin{align*}
    \tilde{R} := \frac{1}{n} \sum_{i=1}^{n} \frac{\tilde{\vP}_0(\vX_i)}{\tilde{\vP}_1(\vX_i)},
\end{align*}
where $\vX_i \sim \vP_1$. Again, $\log \tilde{R}$ is a biased estimator of $\log\left(\frac{Z_0}{Z_1}\right)$, and unless a large number of samples is used, the resulting bias leads to poor guarantees.

Neither of these approaches is effective for our setting due to the inherent biases that significantly affect detection performance. Therefore, in the next section, we introduce Thermodynamic Integration (TI), which provides a more efficient solution.
\subsubsection{Thermodynamic Integration (TI)}
\label{subsec:TI}

{TI} is a technique used in physics to approximate intractable normalized constants of high-dimensional distributions. 
It is based on the observation that it is easier to calculate the ratio of two unknown normalizing constants than it is to calculate the constants themselves. More formally, consider two densities over space $\mathcal{X}$:
\begin{align}
  \vP_i(\vX) &= \frac{\tilde{\vP}_i(\vX)}{Z_i}, \quad Z_i = \int_{\mathcal{X}} \tilde{\vP}(\vX) \,\vdx, \, i \in \{0, 1\}.
\end{align}
To apply TI, we form a continuous family (or ``path'') between $\vP_0(\vX)$ and $\vP_1(\vX)$ via a scalar parameter $\beta \in [0,1]$:
\begin{align}
  &\vP_{\beta}(\vX) = \frac{\tilde{\vP}_{\beta}(\vX)}{Z_{\beta}} = \frac{\tilde{\vP}_1(\vX)^{\beta}\tilde{\vP}_0(\vX)^{{1 - \beta}}}{Z_{\beta}},\\ \quad &Z_{\beta} = \int_{\mathcal{X}} \tilde{\vP}_{\beta}(\vX) \,\vdx, \quad \beta \in [0, 1].  \label{eq:pi_z}
\end{align}

The central identity that allows us to compute the ratio $\log(Z_1/Z_0)$ is derived as follows. Assuming we can exchange integration with differentiation:
\begin{align}
  \frac{\partial \log Z_{\beta}}{\partial \beta} &= \frac{1}{Z_{\beta}} \frac{\partial}{\partial \beta}Z_{\beta} \nonumber \\
  &= \frac{1}{Z_{\beta}} \frac{\partial}{\partial \beta} \int \tilde{\vP}_{\beta}(\vX) \,\vdx \nonumber \\
  &= \int \frac{1}{Z_{\beta}} \frac{\partial}{\partial \beta} \tilde{\vP}_{\beta}(\vX) \,\vdx \nonumber \\
  &= \int \frac{\tilde{\vP}_{\beta}(\vX)}{Z_{\beta}}\frac{\partial}{\partial \beta} \log \tilde{\vP}_{\beta}(\vX) \,\vdx\nonumber\\
  &= \int {{\vP}_{\beta}(\vX)}\frac{\partial}{\partial \beta} \log \tilde{\vP}_{\beta}(\vX) \,\vdx,
\end{align}
which directly implies
\begin{align}
  \frac{\partial \log Z_{\beta}}{\partial \beta} &= \E_{\vX\sim\vP_{\beta}}\big[U^{\prime}_{\beta}(\vX)\big], \label{eq:background/thermo/dZdbeta}
\end{align}
where the quantity $U_{\beta}(\vX) = \log \tilde{\vP}_\beta(\vX)$ is referred to as the ``potential'' in statistical physics and $U^{\prime}_{\beta}(\vX) := \frac{\partial}{\partial \beta} U_{\beta}(\vX)$.
\begin{remark}
For exchange of integration and differentiation we assume we have the following conditions:
\begin{enumerate}
    \item $\tilde{\vP}_{\beta}(\vX)$ is a integrable function of $\beta$ for each $\vX \in \mathcal{X}$.
    \item For almost all $\beta \in [0,1]$, the partial derivative $\frac{\partial}{\partial \beta}\tilde{\vP}_{\beta}(\vX)$ exists for all $\vX \in \mathcal{X}$.
    \item There is an integrable function $f: \beta \rightarrow [0,1]$ such that $|\frac{\partial}{\partial \beta}\tilde{\vP}_{\beta}(\vX)| \leq f(\beta)$ for all $x \in X$ and almost every $\beta \in [0,1]$.
\end{enumerate}

\end{remark}
The variable $\beta$ can be interpreted as the inverse temperature parameter. Because one can typically compute $\log \tilde{\vP}_\beta(\vX)$, \eqref{eq:background/thermo/dZdbeta} allows us to exchange the first derivative of something we cannot compute with an expectation over something we can compute. Then, to calculate the ratio $\log(Z_1/Z_0)$ we integrate out $\beta$ on both sides of \eqref{eq:background/thermo/dZdbeta}:
\begin{align}
  \int_0^1 \frac{\partial \log Z_\beta}{\partial \beta} \,d\beta &= \int_0^1 \E_{\vX\sim\vP_{\beta}}\big[U^{\prime}_{\beta}(\vX)\big] \,d\beta,
\end{align}
which, via the Lebesgue differentiation theorem, results in
\begin{align}
  \log(Z_1) - \log(Z_0) &= \int_0^1 \E_{\vX\sim\vP_{\beta}}\big[U^{\prime}_{\beta}(\vX)\big] \,d\beta.
\end{align}
 This means
\begin{align}
  \log(Z_1) - \log(Z_0) &= \E_{{\beta}\sim U([0,1])} \E_{\vX\sim\vP_{\beta}}\big[U^{\prime}_{\beta}(\vX)\big] .
\end{align}
Therefore, $U^{\prime}_{\beta}(\vX)$ is an unbiased estimator of $\log(Z_1) - \log(Z_0)$, when ${\beta}\sim \text{Uniform}([0,1])$ and $\vX|\beta\sim \vP_{\beta}$.
\subsubsection{Computing the Expectations}

In this part, we show how to implement $\E_{\vX\sim\vP_{\beta}}\big[U^{\prime}_{\beta}(\vX)\big]$. Important sampling is necessary for efficient implementation of the expectation in thermodynamic integration  \cite{masrani2023advancing}.
To do that, we use the following identity:
\begin{align}
\E_{\vX\sim\vP_{\beta}}\big[U^{\prime}_{\beta}(\vX)\big] = \E_{\vX\sim\vP_{0}} \left[\frac{w(\vX)^\beta U^{\prime}_{\beta}(\vX)}{\E_{\vX\sim\vP_{0}} [w(\vX)^\beta]}\right],
\end{align}
where $w(\vX) = \frac{\tilde{\vP}_{1}(\vX)}{\tilde{\vP}_{0}(\vX)}$. By substituting $w(\vX)$, we can show that
\begin{align}
\E_{\vX\sim\vP_{0}} [w(\vX)^\beta] = \frac{Z_\beta}{Z_0},
\end{align}
and
\begin{align}
\E_{\vX\sim\vP_{0}} \left[w(\vX)^\beta U^{\prime}_{\beta}(\vX)\right] = \frac{Z_\beta}{Z_0}\E_{\vX\sim\vP_{\beta}}\big[U^{\prime}_{\beta}(\vX)\big].
\end{align}

This will help us in the implementation of the estimator. We sample $K$ samples $\{\vX_1, \vX_2, \dots, \vX_K\} \sim \vP_0$ and use the following identity to compute the expectation:
\begin{align}
\E_{\vX\sim\vP_{\beta}}\big[U^{\prime}_{\beta}(\vX)\big] \approx \sum_{i=1}^{K}\frac{w(\vX_i)^\beta U^{\prime}_{\beta}(\vX_i)}{\sum_{j=1}^{K} w(\vX_j)^\beta}.
\end{align}
Since $w(\vX)$ is the same for different $\beta$, we can reuse the samples $\{\vX_1, \vX_2, \dots, \vX_K\}$ for different $\beta$, reducing the number of samples needed.
\subsubsection{Variance of the Estimator}
In this section, we derive an upper bound for the variance of the estimator $U^{\prime}_{\beta}(\vX)$. 
The key steps in deriving this upper bound are:
\begin{itemize}
    \item Using the law of total variance to express $\Var[U^{\prime}_{\beta}(\vX)]$ in terms of the variance under $\vP_\beta(\vX)$ and the variance of the expectation of $U^{\prime}_{\beta}(\vX)$ under $\vP_\beta(\vX)$.
    \item Recognizing that the variance under $\vP_\beta(\vX)$ is the second derivative of $\log Z_\beta$ with respect to $\beta$, which is an increasing function of $\beta$.
    \item Utilizing the expressions for $\frac{\partial \log Z_\beta}{\partial \beta}$ evaluated at $\beta = 0$ and $\beta = 1$, which involve the KL divergences between $\vP_1$ and $\vP_0$, and the log ratio of the partition functions.
    \item Bounding the expected value of the squared derivative of $\log Z_\beta$ with respect to $\beta$ using the increasing property and the expressions at the endpoints $\beta = 0$ and $\beta = 1$.
\end{itemize}
We will present the preliminary lemmas required to prove the Theorem \ref{thm:4}.

\begin{lemma}\label{lem:4}
The second derivative of the log partition function $\log Z_{\beta}$ with respect to $\beta$ is the variance of $U^{\prime}_{\beta}(\vX)$ under the distribution $\vP_{\beta}(\vX)$, i.e.,
\begin{align}
   \frac{\partial^2 \log Z_{\beta}}{\partial \beta^2} &= \Var_{\vP_{\beta}(\vX)}[U^{\prime}_{\beta}(\vX)] \geq 0.
\end{align}
This implies that $\frac{\partial \log Z_{\beta}}{\partial \beta}$ is an increasing function of $\beta$.
\end{lemma}

\begin{lemma}
The derivative of the log partition function evaluated at $\beta = 1$ and $\beta = 0$ is related to the Kullback-Leibler (KL) divergence between the distributions $\vP_1$ and $\vP_0$, and the log ratio of the partition functions $Z_1$ and $Z_0$, as follows:
\begin{align}
   \frac{\partial \log Z_{\beta}}{\partial \beta}\bigg|_{\beta=1} &= \mathbb{D}_{\text{KL}}(\vP_1, \vP_0) + \log\left(\frac{Z_1}{Z_0}\right), \\
   \frac{\partial \log Z_{\beta}}{\partial \beta}\bigg|_{\beta=0} &= \mathbb{D}_{\text{KL}}(\vP_0, \vP_1) + \log\left(\frac{Z_1}{Z_0}\right).
\end{align}
\end{lemma}

\begin{lemma}\label{lem:varbound}
The variance of the estimator $U^{\prime}_{\beta}(\vX)$ can be upper bounded as follows:
\begin{equation}
 \resizebox{1\hsize}{!}{$
\begin{aligned}
   \Var[U^{\prime}_{\beta}(\vX)]& \leq \mathbb{D}_{\text{KL}}(\vP_1, \vP_0) + \mathbb{D}_{\text{KL}}(\vP_0, \vP_1) + 2\log\left(\frac{Z_1}{Z_0}\right) \notag\\
   &\quad + \E_{\beta\sim U([0,1])} \left[\left(\frac{\partial \log Z_\beta}{\partial \beta}\right)^2\right] - \log^2\left(\frac{Z_1}{Z_0}\right).
   \label{eq:var}
\end{aligned}
$}
\end{equation}
\end{lemma}

\begin{lemma}\label{lem:der2}
The following upper bound holds for the expected value of the squared derivative of the log partition function:
\begin{align*}
   &\E_{\beta\sim U([0,1])} \left[\left(\frac{\partial \log Z_\beta}{\partial \beta}\right)^2\right] \\
   &\leq 3\left(\mathbb{D}_{\text{KL}}(\vP_1, \vP_0) + \mathbb{D}_{\text{KL}}(\vP_0, \vP_1) + 2\left|\log\left(\frac{Z_1}{Z_0}\right)\right|\right)^2.
\end{align*}
\end{lemma}

\begin{theorem}\label{thm:4}
Combining the results from the previous lemmas, we obtain the following upper bound for the variance of the estimator $U^{\prime}_{\beta}(\vX)$:
    \begin{equation}\label{eq:finalit}
 \resizebox{1\hsize}{!}{$
\begin{aligned}
     \Var[U^{\prime}_{\beta}(\vX)] &\leq 3\left(\mathbb{D}_{\text{KL}}(\vP_1, \vP_0) + \mathbb{D}_{\text{KL}}(\vP_0, \vP_1) + 2\left|\log\left(\frac{Z_1}{Z_0}\right)\right|\right)^2 \notag\\
    &\quad + \left(\mathbb{D}_{\text{KL}}(\vP_1, \vP_0) + \mathbb{D}_{\text{KL}}(\vP_0, \vP_1) + 2\left|\log\left(\frac{Z_1}{Z_0}\right)\right|\right).
\end{aligned}
$}
\end{equation}

\end{theorem}

The resulting upper bound in \eqref{eq:finalit} provides a computable expression for the variance of the estimator $U^{\prime}_{\beta}(\vX)$ in terms of the KL divergences between the distributions $\vP_1$ and $\vP_0$, and the log ratio of their partition functions. This bound can be useful for analyzing the properties of the estimator and its convergence behavior. The proof of the Theorem \ref{thm:4} is in the supplementary materials.

\begin{figure*}[t]
    \centering
    \begin{subfigure}[b]{0.49\linewidth}
        \centering
        \includegraphics[width=\linewidth]{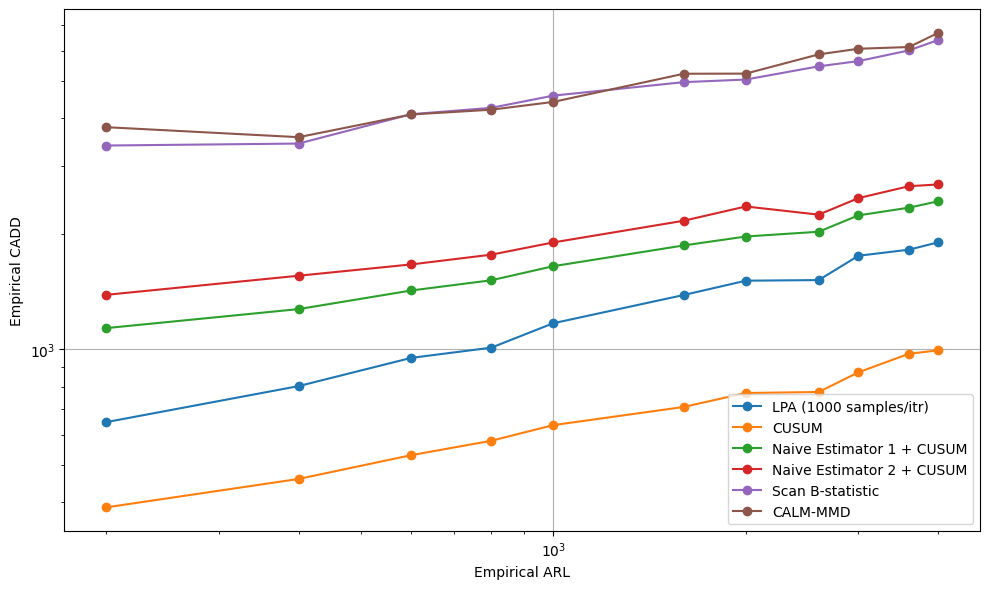}
        \caption{Empirical CADD against Empirical ARL for the EXP.}
        \label{fig:empirical_cadd_vs_arl_exp}
    \end{subfigure}
    \hfill
    \begin{subfigure}[b]{0.49\linewidth}
        \centering
        \includegraphics[width=\linewidth]{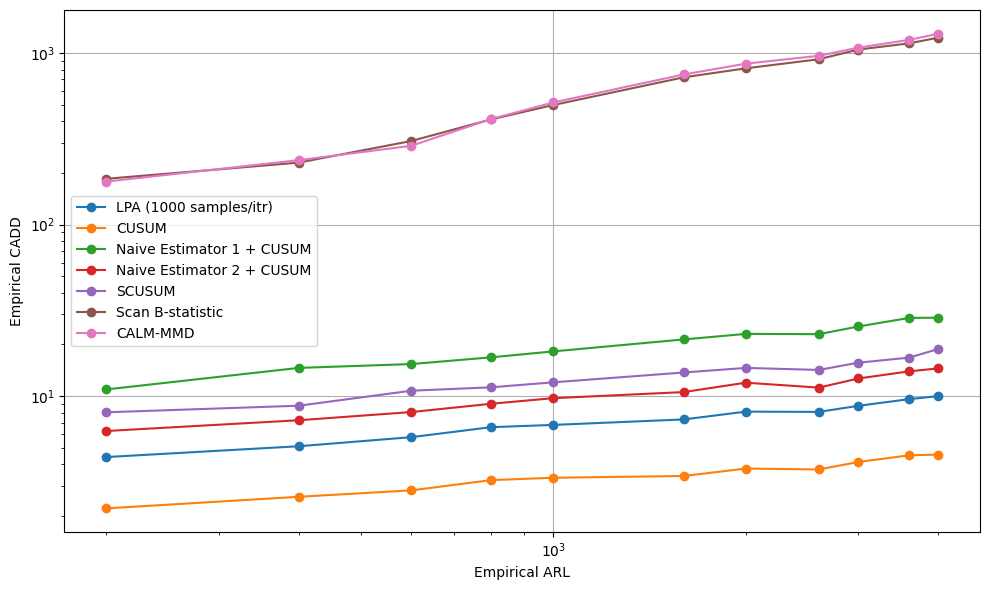}
        \caption{Empirical CADD against Empirical ARL for MVN.}
        \label{fig:empirical_cadd_vs_arl_mvn}
    \end{subfigure}
    \caption{Comparison of Empirical CADD vs. ARL for EXP (left) and MVN (right).}
    \label{fig:combined_cadd_vs_arl}
\end{figure*}

\section{NUMERICAL RESULTS}
\label{sec:results}

In this section, we conduct comprehensive numerical experiments on synthetic data to evaluate the performance of our proposed method in comparison with established change detection algorithms for unnormalized data. Specifically, we compare against SCUSUM \cite{wu2023quickest}, Scan B-statistic \cite{li2019scan}, CALM-MMD \cite{cobb2022sequential}, and two naive estimation approaches:

\begin{itemize}
    \item The first naive method estimates $ Z_0 $ as $ \hat{Z}_0 $ and $ Z_1 $ as $ \hat{Z}_1 $, then applies $ \log\frac{\hat{Z}_0}{\hat{Z}_1} $ within the CUSUM framework.
    \item The second naive method follows the approach described in Equation \eqref{eq:naive2} of the paper.
\end{itemize}
Throughout the experiments, we refer to our proposed algorithm, presented in Algorithm~\ref{alg:LPA_CUSUM}, as LPA. The score-based method from \cite{wu2023quickest} is denoted as SCUSUM, while the classical Cumulative Sum method from \cite{page1955test} is referred to as CUSUM. Additionally, we refer to the first naive estimator as Naive Estimator 1 and the second naive estimator as Naive Estimator 2.
\paragraph{Multivariate Normal Distribution (MVN).}
We consider synthetic data generated from a 10-dimensional multivariate normal distribution. The pre-change distribution is characterized by a mean vector $\boldsymbol{\mu} = \boldsymbol{0}$ and a covariance matrix. For the post-change scenario, we investigate both mean shifts and covariance changes. Specifically, we set the post-change mean as $\boldsymbol{\mu}_1 = \boldsymbol{1}$ and modify the covariance components with a small change.
\paragraph{Boltzmann Distribution (EXP).}  
We also consider the Boltzmann distribution, which models the probability of a system being in a certain state as a function of that state and the system's temperature. Our focus is on the problem of temperature change detection. The distribution is given by:  
\begin{equation*}  
    \vP(\vX) = \frac{1}{T} \exp\left(-\frac{\vX}{T} \right),  
\end{equation*}  
where $T$ represents the temperature of the system.

\paragraph{Simulation Setting.} We simulate LPA with 1000 samples in each iteration to estimate the log partition function. 
We generate synthetic data streams from two sources: the 10-dimensional Multivariate Normal Distribution (MVN) and a Boltzmann distribution. For MVN, we specify different pre- and post-change settings to test the sensitivity of our methods. The pre-change data follows a 10-dimensional multivariate normal distribution with mean vector $\boldsymbol{\mu} = \boldsymbol{0}$ and a predefined covariance matrix. The post-change data is generated with a shifted mean $\boldsymbol{\mu}_1 = \boldsymbol{1}$ while retaining a close covariance structure ($\|\Sigma_0 -  \Sigma_1\|_{2} \leq 1$). We generate $\Sigma_0$ randomly and ensure that it is a proper covariance matrix. Then, we set $ \Sigma_1 = \Sigma_0 + \Sigma_\epsilon$, where $\Sigma_\epsilon$ is a random matrix that ensures $ \Sigma_1$ is a proper covariance matrix and satisfies $\|\Sigma_\epsilon\|_2 \leq 1$.  The specific structure of the covariance matrix can be found in the supplementary materials. For Boltzmann distribution , we use $T=1$ for pre-change distribution and $T=1.2$ for post-change distribution.

The performance of the algorithms is evaluated in terms of empirical Average Run Length (ARL) and empirical Conditional Average Detection Delay (CADD), defined by $\mathbb{E}_{\infty}[T]$ and $\mathbb{E}_{\nu}[T - \nu \mid T \geq \nu]$, respectively. The empirical ARL measures the average time until a false alarm occurs when no change is present, with larger values indicating fewer false alarms. The empirical CADD assesses the detection delay after a change occurs, with smaller values signifying faster detection. All results for empirical ARL and CADD are presented on a logarithmic scale. In each experiment, the change point is set at $\nu = 500$, with a total stream length fixed at $10,000$ to ensure sufficient data for detection. The ARL values are chosen from the range $200$ to $10,000$.

Note that for the Boltzmann model, the SCUSUM algorithm is unable to detect the change because the score of the distribution is a constant. As a result, SCUSUM is ineffective in this setting, which is why we did not include it as a baseline. 

We observe that the CADD for the LPA algorithm decreases as the number of samples for the log partition function estimation increases. Furthermore, as the number of samples increases, the performance of the LPA algorithm approaches that of the CUSUM algorithm, which is consistent with our theoretical results. Also, with sufficiently large samples for estimating the partition function, LPA performs better than SCUSUM.
\section*{Acknowledgements}
This work was supported by the US National Science Foundation
under the award numbers 2334898, 2334897, and ECCS-2335876.


\bibliographystyle{apalike}
\bibliography{refs}
\newpage
\appendix
\onecolumn

\section*{SUPPLEMENTARY MATERIALS}

\section{Proof of Lemma \ref{lemma:tn_y} }
 
     The proof utilizes Jensen's Inequality to establish the inequality between the expectations of exponential functions of ${\vT}_{i,n}$ and $\vY$. Consider the definition of ${\vT}_{i,n}$ as the average of $n$ independent and identically distributed random variables $\vY_{1,n}, \vY_{2,n}, \dots, \vY_{i,n}$, denoted as $\vY$. Then, we have
$$\E_{0}[\exp(\gamma {\vT}_{i,n})] = \left(\E_{0}[\exp(\frac{1}{i}\gamma \vY)]\right)^i \leq \E_{0}[\exp(\gamma \vY)],$$
where the first equality comes from independence and the last inequality follows directly from Jensen's Inequality, as the exponential function is a convex function.

\section{Proof of Lemma \ref{lem:existanceof good gamma}}
 The first part of the lemma follows directly from a previously established result (lemma 2 in  \cite{wu2023quickest}), which provides the existence of a positive $\gamma$ satisfying the inequality in equation \eqref{eq:samegamma1}. Leveraging the fact that $\vY_i$s are independent and identically distributed (i.i.d.), we proceed to analyze the expectation of the product of exponential functions involving ${\vT}_{i,n}$ and $\vX_n$. Specifically,
  \begin{align}
\E_{0}&[\exp(\gamma\log\frac{\tilde \vP_1(\vX_n)}{\tilde \vP_0(\vX_n)}+\gamma{\vT}_{i,n})]\\
&=\E_{0}[\exp(\gamma\log\frac{\tilde \vP_1(\vX_n)}{\tilde \vP_0(\vX_n)})].\E_{0}[\exp(\gamma{\vT}_{i,n})]\\
&\leq \E_{0}[\exp(\gamma\log\frac{\tilde \vP_1(\vX_n)}{\tilde \vP_0(\vX_n)})].\E_{0}[\exp(\gamma \vY)] \\
&\leq \E_{0}[\exp(\gamma\log\frac{\tilde \vP_1(\vX_n)}{\tilde \vP_0(\vX_n)}+\gamma \vY)] \leq 1,
  \end{align}
  where we used independence of ${\vT}_{i,n}$ and $\vX_n$ for the first equality, and lemma \ref{lemma:tn_y} for the second inequality.

\section{Proof of Theorem \ref{thm:1} }\label{sec:proofthm1}
    First, because of lemma \ref{lem:existanceof good gamma}, there exist a $\gamma$ that for all $n\geq 1$, equation \ref{eq:goodgamma} holds. Now, let's define
\begin{equation}
    Z_m^t := \gamma\sum_{n=m}^{t} \log\frac{\tilde \vP_1(\vX_n)}{\tilde \vP_0(\vX_n)}+\gamma\sum_{n=m}^{t}{\vT}_{i,n},
\end{equation}
and define the stopping times
$$
\eta_{\ell+1} \triangleq \inf \left\{k \geq \eta_{\ell}+1: Z_{\eta_{\ell}+1}^k<0\right\},
$$
for $\ell \geq 0$ where $\eta_0 \triangleq 0$. Since $\E_{0}\left[\exp(\gamma\log\frac{\tilde \vP_1(\vX_n)}{\tilde \vP_0(\vX_n)}+\gamma{\vT}_{i,n})\right] \leq 1$, we have that $\left\{\exp \left(Z_n^k\right), \mathcal{F}_k, k \geq n\right\}$ is a non-negative supermartingale. Then, on events $\left\{\eta_{\ell}<\infty\right\}$, we have

\begin{align}\label{eq1}
& \mathbb P\left(\max _{t>\eta_{\ell}} Z_{\eta_{\ell}+1}^t \geq h \mid \mathcal{F}_{\eta_{\ell}}\right) \\
& \quad \leq e^{-h} \E_{0}\left[\exp(\gamma\log\frac{\tilde \vP_1(\vX_{\eta_{\ell}+1})}{\tilde \vP_0(\vX_{\eta_{\ell}+1})}+\gamma{\vT}_{\eta_{\ell}+1})\mid \mathcal{F}_{\eta_{\ell}}\right]\\& \quad\leq e^{-h},
\end{align}

where the first inequality is the maximal inequality for nonnegative supermartingales, and the second inequality is due to independence and the lemma condition.

Now, let's define $N_z$ as a lower bound on the number of zero crossings of the test statistic before a threshold crossing in the sense that
$$
N_z \triangleq \inf \left\{\ell \geq 0: \eta_{\ell}<\infty \text { and } \max _{t>\eta_{\ell}} Z_{\eta_{\ell}+1}^t \geq h\right\}.
$$
Then, for any $\ell \geq 0$, we have that
$$
\begin{aligned}
\vP_0 & \left(N_z>\ell\right) \\
& =\vP_0\left(N_z>\ell-1 \text { and } N_z>\ell\right) \\
& =E\left[\mathbb{I}\left\{N_z>\ell-1\right\} \vP_0\left(N_z>\ell \mid \mathcal{F}_{\eta_{\ell}}\right)\right] \\
& =E\left[\mathbb{I}\left\{N_z>\ell-1\right\} \vP_0\left(\max _{t>\eta_{\ell}} Z_{\eta_{\ell}+1}^t<h \mid \mathcal{F}_{\eta_{\ell}}\right)\right] \\
& \geq \vP_0\left(N_z>\ell-1\right)\left(1-e^{-h}\right),
\end{aligned}
$$
where the second equality follows from the tower property of conditional expectation by noting that $\mathbb{I}\left\{N_z \geq \ell\right\}$ is $\mathcal{F}_{\eta_{\ell}}$ measurable, the third equality follows from the definition of $N_z$, and the last line follows from \eqref{eq1}. Since $N_z$ is non-negative and $e^{-h} \leq 1$ for $h>0$, it follows that
$$
\begin{aligned}
\E_{0}\left[\tau\right]\geq\E_{0}\left[N_z\right] & \geq \sum_{\ell=0}^{\infty} \vP_0\left(N_z>\ell\right) \\
& \geq \sum_{\ell=0}^{\infty}\left(1-e^{-h}\right)^{\ell} \\
& \geq e^h.
\end{aligned}
$$

\section{Proof of the Theorem \ref{thm:delay}}\label{sec:proof of thmdelay}
This Proof is inspired by the proof of Theorem 8.2.6 of  \cite{tartakovsky2014sequential}.
\begin{lemma}\label{lemma:delay}
Define
\begin{equation}
Z_m^t = \gamma\sum_{n=m}^{t} \log\frac{\tilde{\vP}_1(\vX_n)}{\tilde{\vP}_0(\vX_n)} + \gamma\sum_{n=m}^{t}{\vT}_{i,n}.
\end{equation}

Assuming that the change occurs at $\nu$, the weak law of large numbers can be applied, and for $m \geq \nu$, we obtain
\begin{align*}
&\lim_{k \to \infty} \mathbb{P}^{\vP_1,\vP_0}_{\nu} \left(
\left| k^{-1} Z_{m}^{m+k-1} - \gamma \mathbb{D}_{\text{KL}}(\vP_1, \vP_0) \right| \geq \delta \right) = 0.
\end{align*}
\end{lemma}

\begin{proof}

The key steps in the proof are:
\begin{itemize}
    \item Expressing $Z_m^t$ as the sum of the log-likelihood ratios and a martingale difference sequence $\{{\vT}_{i,n}\}$.
        \item Bounding the variance of $k^{-1} Z_{m}^{m+k-1}$ by considering the variance of the martingale difference sequence, which goes to zero as $k$ increases.

    \item Showing that the mean of $k^{-1} Z_{m}^{m+k-1}$ converges to $\gamma \mathbb{D}_{\text{KL}}(\vP_1, \vP_0)$ using the weak law of large numbers for the log-likelihood ratios.
\end{itemize}

We start by showing that the mean of $k^{-1} Z_{m}^{m+k-1}$ is $\mathbb{D}_{\text{KL}}(\vP_1, \vP_0)$ and then we show that the variance of $k^{-1} Z_{m}^{m+k-1}$ goes to zero.

For $n \geq \nu$, we have
\begin{align}
\E_\nu^{\vP_1,\vP_0}&\left[\gamma\left(\log\frac{\tilde{\vP}_1(\vX_n)}{\tilde{\vP}_0(\vX_n)} + {\vT}_{i,n}\right)\right]  = \E_\nu^{\vP_1,\vP_0}\left[\gamma \log\frac{\vP_1(\vX_n)}{\vP_0(\vX_n)}\right] = \gamma \mathbb{D}_{\text{KL}}(\vP_1, \vP_0).
\end{align}

For the variance, we have
\begin{equation}
\log\frac{\tilde{\vP}_1(\vX_n)}{\tilde{\vP}_0(\vX_n)} + {\vT}_{i,n} = \log\frac{\vP_1(\vX_n)}{\vP_0(\vX_n)} + ({\vT}_{i,n} - \log \frac{Z_0}{Z_1}).
\end{equation}
So, we focus on bounding the variance of the second part of $k^{-1} Z_{m}^{m+k-1}$ since the variance of the first part goes to zero. We have
\begin{align}
\notag&\Var\left(\frac{1}{k}\sum_{n=m}^{m+k-1}\left({\vT}_{i,n} - \log \frac{Z_0}{Z_1}\right)\right) =\frac{1}{k^2}\sum_{n=m}^{m+k-1} \Var\left({\vT}_{i,n} - \log \frac{Z_0}{Z_1}\right)\leq \frac{1}{k^2}k\sigma^2 = \frac{1}{k}\sigma^2,
\end{align}
where $\sigma^2$ is the (finite) variance of ${\vT}_{i,n} - \log \frac{Z_0}{Z_1}$.
\end{proof}
\begin{proof}[Proof of Theorem \ref{thm:delay}]

Now, consider any arbitrary $\delta \in (0, 1)$ and define the integer $$k_c = \lfloor\frac{ \log h }{\gamma\mathbb{D}_{\text{KL}}(\vP_1, \vP_0)(1 - \delta)} \rfloor,$$ for any $\log h > 0$, where $\lfloor \cdot \rfloor$ denotes the floor function.

Based on Lemma \ref{lemma:delay}, we have
\begin{equation}
\sup_{1 \leq \nu \leq t} \vQ_\nu\left(Z_t^{t+k_c-1} < \log h\right) < \delta,
\end{equation}
for sufficiently large $h$, and any $t \geq 1$. For sufficiently large $h$,
\begin{align}
&\notag\operatorname{ess} \sup \vQ_\nu\left(((\tau-\nu+1)^{+} > tk_c \mid \mathcal{F}_{\nu-1}\right)\leq \prod_{j=1}^t \vQ_\nu\left(Z_{\nu+(j-1)k_c}^{\nu+jk_c-1} < \log h\right) \leq \delta^t,
\end{align}
which means
\begin{equation}
\begin{aligned}
\operatorname{ess} \sup \E_\nu\left[\left(((\tau-\nu+1)^{+}\right)
\mid \mathcal{F}_{\nu-1}\right] &=\operatorname{ess} \sup \int_0^{\infty} \vQ_\nu\left(((\tau-\nu+1)^{+} > x \mid \mathcal{F}_{\nu-1}\right) \mathrm{d} x \\
&\leq \sum_{t=0}^{\infty} k_c \times \operatorname{ess} \sup \vQ_\nu\left(((\tau-\nu+1)^{+} > tk_c \mid \mathcal{F}_{\nu-1}\right) \\
&\leq k_c\sum_{t=0}^{\infty} \delta^t.
\end{aligned}
\end{equation}
Therefore,
\begin{align}
\operatorname{ess} \sup \E_\nu\left[\left((\tau-\nu+1)^{+}\right) \mid \mathcal{F}_{\nu-1}\right] \notag\leq (1+o(1))\left(\frac{\log h}{\gamma \mathbb{D}_{\text{KL}}(\vP_1, \vP_0)}\right) \frac{\sum_{t=0}^{\infty} \delta^t}{(1-\delta)}.
\end{align}  
\end{proof}
\section{Proof of the Theorem \ref{thm:asy_opt} }

Let us define 
\begin{align*}
M_\vY(\gamma) := \E_0\left[\exp\left(\gamma \left(\vY - \log\frac{Z_0}{Z_1}\right)\right)\right],
\end{align*}
where $\vY$ is a random variable representing the output of the oracle $\mathcal{A}$. Then, we have
\begin{align}
F(\gamma, \frac{1}{i}) :=& \E_0\left[\exp\left(\gamma\log\frac{\vP_1(\vX)}{\vP_0(\vX)} + \gamma{\vT}_{i,n} -
\gamma\log\frac{Z_0}{Z_1}\right)\right] \notag \\
&= \E_0\left[\exp\left(\gamma\log\frac{\vP_1(\vX)}{\vP_0(\vX)}\right) \left(M_\vY\left(\frac{\gamma}{i}\right)\right)^i\right].
\end{align}
Let us define a function that generalizes the previous one by replacing the second variable $i$ with a continuous variable $\beta$,
\begin{align}
F(\gamma, \beta) := \E_0\left[\exp\left(\gamma\log\frac{\vP_1(\vX)}{\vP_0(\vX)}\right) \left(M_\vY(\gamma\beta)\right)^{\frac{1}{\beta}}\right].
\end{align}
Using Taylor's series approximation, we have the following:
\begin{align}
F(\gamma, \beta) = F(1,0) + g_{\gamma}(1,0)(\gamma - 1) + g_{\beta}(1,0)(\beta)+er(\gamma, \beta),
\end{align}
where $g_{\gamma}(\gamma,\beta):=\frac{\partial F(\gamma,\beta)}{\partial \gamma}$, and $g_{\beta}(\gamma,\beta):=\frac{\partial F(\gamma,\beta)}{\partial \beta}$.

We have
\begin{align*}
g_{\beta}(\gamma, \beta) = \mathbb{E}_0\left[\exp\left(\gamma \log\frac{\vP_1(\vX)}{\vP_0(\vX)}\right) \left(M_\vY(\gamma \beta)\right)^{\frac{1}{\beta}} \left( \frac{-\log(M_\vY(\gamma \beta))}{\beta^2} + \frac{\gamma M_\vY'(\gamma \beta)}{\beta M_\vY(\gamma \beta)} \right)\right],
\end{align*}
and 
\begin{align*}
    g_{\gamma}(\gamma, \beta) = \mathbb{E}_0\left[\log\frac{\vP_1(\vX)}{\vP_0(\vX)} \exp\left(\gamma \log\frac{\vP_1(\vX)}{\vP_0(\vX)}\right) \left(M_\vY(\gamma \beta)\right)^{\frac{1}{\beta}} + \exp\left(\gamma \log\frac{\vP_1(\vX)}{\vP_0(\vX)}\right) \left(M_\vY(\gamma \beta)\right)^{\frac{1}{\beta} - 1} \cdot  M_\vY'(\gamma \beta)\right].
\end{align*}

Since $F(1,0) = 1$, $g_{\gamma}(1,0) = \mathbb{D}_{\text{KL}}(\vP_1, \vP_0)$, and $g_{\beta}(1,0) = \frac{\sigma^2}{2}$, where $\sigma^2$ is the variance of $\vY$, if we want to find the relationship between $\gamma$ and $\beta$ that makes 
\begin{align}
F(1,0) + g_\gamma(1,0) (\gamma - 1) + g_\beta(1,0) (\beta)= 1-\epsilon\beta,
\end{align}
this results in
\begin{align}
1 + \mathbb{D}_{\text{KL}}(\vP_1, \vP_0) (\gamma - 1) + \frac{\sigma^2}{2}(\beta) = 1-\epsilon\beta,
\end{align}
and consequently, we should let
\begin{align}\label{eq:choiceofgamma}
\gamma_{0}= 1 - \frac{\sigma^2}{2i \mathbb{D}_{\text{KL}}(\vP_1, \vP_0)}-\frac{2\epsilon}{2i \mathbb{D}_{\text{KL}}(\vP_1, \vP_0)}.
\end{align}
Equivalently, we can rewrite this as
\begin{align}\label{eq:choiceofbeta1}
\frac{1}{i} = 2(1 - \gamma_{0})\frac{\mathbb{D}_{\text{KL}}(\vP_1, \vP_0)}{\sigma^2+2\epsilon}.
\end{align}

Now, we want to show that $F(\gamma_{0}, \frac{1}{i}) \leq 1$ for large enough $i$. Let us define
\begin{align}\label{eq:choiceofgamma}
\beta_0 := 2(1 - \gamma)\frac{\mathbb{D}_{\text{KL}}(\vP_1, \vP_0)}{\sigma^2+2\epsilon}.
\end{align}
and the function $h$ as
\begin{align*}
    h(\gamma):=F(\gamma,\beta_0).
\end{align*}
We have that 
\begin{align*}
    \frac{\partial h(\gamma)}{\partial \gamma} = \frac{\partial F(\gamma, \beta_0)}{\partial \gamma} + \frac{\partial F(\gamma, \beta_0)}{\partial \beta} \cdot \frac{\partial \beta_0}{\partial \gamma},
\end{align*}
where
\begin{align*}
    \frac{\partial \beta_0}{\partial \gamma} = -2 \frac{\mathbb{D}_{\text{KL}}(\vP_1, \vP_0)}{\sigma^2 + 2\epsilon}.
\end{align*}
For small enough $\delta$, when $1-\delta < \gamma \leq 1$, we have that 
\begin{align*}
    g_{\beta}(\gamma, \beta_0)= g_{\beta}(1, 0)+\alpha_1=\frac{\sigma^2}{2}+\alpha_1,
\end{align*}
and 
\begin{align*}
    g_{\gamma}(\gamma, \beta_0)= g_{\gamma}(1, 0)+\alpha_2=\mathbb{D}_{\text{KL}}(\vP_1, \vP_0)+\alpha_2,
\end{align*}
\begin{align*}
    \frac{\partial h(\gamma)}{\partial \gamma} = \mathbb{D}_{\text{KL}}(\vP_1, \vP_0) + \frac{\sigma^2}{2} \left(-2 \frac{\mathbb{D}_{\text{KL}}(\vP_1, \vP_0)}{\sigma^2 + 2\epsilon}\right) + \alpha_1 + \alpha_2 > 0.
\end{align*}
Note that in the above, if we make $\delta$ small enough, both $\alpha_1$ and $\alpha_2$ are negligible compared to the rest, and that is why the derivative is positive. This means
\begin{align*}
    1 = F(1,0) = h(1) \geq h(\gamma).
\end{align*}

Therefore, for $i > \frac{\sigma^2 + 2\epsilon}{2\delta \mathbb{D}_{\text{KL}}}$, since the derivative is positive, we have $F(\gamma_{0}, \frac{1}{i}) \leq F(1, 0)=1$.

\section{Proof of Theorem \ref{thm:4}}\label{sec:proofthm3}

    By combining Lemma \ref{lem:varbound} and Lemma \ref{lem:der2}, we arrive at the Theorem \ref{thm:4}.

Using Lemma \ref{lem:varbound}, we have
\begin{align}
   \Var[U'(\vX)] &\leq \mathbb{D}_{\text{KL}}(\vP_1, \vP_0) + \mathbb{D}_{\text{KL}}(\vP_0, \vP_1) + 2\log\left(\frac{Z_1}{Z_0}\right) + \E_{\beta\sim U([0,1])} \left[\left(\frac{\partial \log Z_\beta}{\partial \beta}\right)^2\right] - \log^2\left(\frac{Z_1}{Z_0}\right)
\\&\leq \left(\mathbb{D}_{\text{KL}}(\vP_1, \vP_0) + \mathbb{D}_{\text{KL}}(\vP_0, \vP_1) + 2\left|\log\left(\frac{Z_1}{Z_0}\right)\right|\right) + \E_{\beta\sim U([0,1])} \left[\left(\frac{\partial \log Z_\beta}{\partial \beta}\right)^2\right] 
   \label{eq:varprime}
\end{align}
Now, we plug in the bound in the Lemma \ref{lem:der2}, to obtain 
\begin{align}
     \Var[U'(\vX)] &\leq 3\left(\mathbb{D}_{\text{KL}}(\vP_1, \vP_0) + \mathbb{D}_{\text{KL}}(\vP_0, \vP_1) + 2\left|\log\left(\frac{Z_1}{Z_0}\right)\right|\right)^2 \notag\\
    &\quad + \left(\mathbb{D}_{\text{KL}}(\vP_1, \vP_0) + \mathbb{D}_{\text{KL}}(\vP_0, \vP_1) + 2\left|\log\left(\frac{Z_1}{Z_0}\right)\right|\right).
\end{align}
\section{Proof of Lemma \ref{lem:4}}

We have
\begin{align}
   \frac{\partial^2 \log Z_{\beta}}{\partial \beta^2} &= \frac{-1}{Z_{\beta}^2} \left(\frac{\partial}{\partial \beta}Z_{\beta}\right)^2 + \frac{1}{Z_{\beta}} \left(\frac{\partial^2}{\partial \beta^2}Z_{\beta}\right) \notag\\
   &= -\left(\E_{\vX\sim\vP_{\beta}}\big[U'(\vX)\big]\right)^2 + \E_{\vX\sim\vP_{\beta}}\big[U'(\vX)^2\big] \notag\\&\quad+ \E_{\vX\sim\vP_{\beta}}\big[U''(\vX)\big].
\end{align}
By the choice of $\tilde{\vP}_{\beta}(\vX) = \tilde{\vP}_1(\vX)^{\beta}\tilde{\vP}_0(\vX)^{1 - \beta}$, we have $U''(\vX) = 0$. Therefore,
\begin{align}
   \frac{\partial^2 \log Z_{\beta}}{\partial \beta^2} &= -\left(\E_{\vX\sim\vP_{\beta}}\big[U'(\vX)\big]\right)^2 + \E_{\vX\sim\vP_{\beta}}\big[U'(\vX)^2\big] \\
   &= \Var_{\vP_{\beta}(\vX)}[U'(\vX)] \geq 0.
\end{align}
This means that $\frac{\partial \log Z_{\beta}}{\partial \beta}$ is an increasing function of $\beta$.

\section{Proof of Lemma \ref{lem:varbound}}

Using the law of total variance, the variance of the estimator can be written as
\begin{align}
   \Var[U'(\vX)] &= \E_{\beta\sim U([0,1])} \big[\Var_{\vP_{\beta}(\vX)}[U'(\vX)]\big] \notag\\
   &\quad + \Var \big[\E_{\vP_\beta(\vX)}\left[U'(\vX)\right]\big] \\
   &= \E_{\beta\sim U([0,1])} \left[\frac{\partial^2 \log Z_{\beta}}{\partial \beta^2}\right] \notag\\
   &\quad + \E_{\beta\sim U([0,1])} \left[\left(\E_{\vP_\beta(\vX)}\left[U'(\vX)\right]\right)^2\right] \notag\\
   &\quad - \log^2\left(\frac{Z_1}{Z_0}\right).
\end{align}
Using the fact that $\frac{\partial \log Z_{\beta}}{\partial \beta}$ is an increasing function of $\beta$, and the results from the previous lemmas, we can upper bound the variance as

\begin{align}\label{eq:finalit22}
   \Var[U'(\vX)] &\leq \frac{\partial \log Z_{\beta}}{\partial \beta}\bigg|_{\beta=1} - \frac{\partial \log Z_{\beta}}{\partial \beta}\bigg|_{\beta=0} \notag\\
   &\quad + \E_{\beta\sim U([0,1])} \left[\left(\E_{\vP_\beta(\vX)}\left[U'(\vX)\right]\right)^2\right] \notag\\
   &\quad - \log^2\left(\frac{Z_1}{Z_0}\right) \\
   &\leq \mathbb{D}_{\text{KL}}(\vP_1, \vP_0) + \mathbb{D}_{\text{KL}}(\vP_0, \vP_1) \notag\\
   &\quad + 2\log\left(\frac{Z_1}{Z_0}\right) + \E_{\beta\sim U([0,1])} \left[\left(\frac{\partial \log Z_\beta}{\partial \beta}\right)^2\right] \notag\\
   &\quad - \log^2\left(\frac{Z_1}{Z_0}\right).
\end{align}

\section{Proof of Lemma \ref{lem:der2}}

Since $\frac{\partial \log Z_\beta}{\partial \beta}$ is an increasing function of $\beta$, we can use the following trick to bound the expected value:

\begin{align*}\label{eq:finalit33}
    &\E_{\beta\sim U([0,1])} \left[\left(\frac{\partial \log Z_\beta}{\partial \beta}\right)^2\right] \\
    &= \E_{\beta\sim U([0,1])} \left[\left(\frac{\partial \log Z_\beta}{\partial \beta} - \frac{\partial \log Z_{\beta}}{\partial \beta}\bigg|_{\beta=0} + \frac{\partial \log Z_{\beta}}{\partial \beta}\bigg|_{\beta=0}\right)^2\right] \\
    &\leq \E_{\beta\sim U([0,1])} \left[\left(\frac{\partial \log Z_\beta}{\partial \beta} - \frac{\partial \log Z_{\beta}}{\partial \beta}\bigg|_{\beta=0}\right)^2 + \left(\frac{\partial \log Z_{\beta}}{\partial \beta}\bigg|_{\beta=0}\right)^2\right] \notag\\
    &\quad + 2\E_{\beta\sim U([0,1])} \left[\frac{\partial \log Z_\beta}{\partial \beta} \cdot \frac{\partial \log Z_{\beta}}{\partial \beta}\bigg|_{\beta=0}\right].
\end{align*}

We used the fact that $\left(\frac{\partial \log Z_\beta}{\partial \beta} - \frac{\partial \log Z_{\beta}}{\partial \beta}\big|_{\beta=0}\right) \geq 0$ for all $\beta \in [0, 1]$, and the increasing property of $\frac{\partial \log Z_\beta}{\partial \beta}$ with respect to $\beta$.

Next, we can bound the first term as follows:
\begin{align*}
    &\E_{\beta\sim U([0,1])} \left[\left(\frac{\partial \log Z_\beta}{\partial \beta} - \frac{\partial \log Z_{\beta}}{\partial \beta}\bigg|_{\beta=0}\right)^2\right] \\
    &\leq \left(\frac{\partial \log Z_{\beta}}{\partial \beta}\bigg|_{\beta=1} - \frac{\partial \log Z_{\beta}}{\partial \beta}\bigg|_{\beta=0}\right)^2 \\
    &= \left(\mathbb{D}_{\text{KL}}(\vP_1, \vP_0) + \mathbb{D}_{\text{KL}}(\vP_0, \vP_1) + 2\log\left(\frac{Z_1}{Z_0}\right)\right)^2.
\end{align*}

For the second term, we have:
\begin{align*}
    &\E_{\beta\sim U([0,1])} \left[\frac{\partial \log Z_\beta}{\partial \beta} \cdot \frac{\partial \log Z_{\beta}}{\partial \beta}\bigg|_{\beta=0}\right] \\
    &= \left(\frac{\partial \log Z_{\beta}}{\partial \beta}\bigg|_{\beta=1} - \frac{\partial \log Z_{\beta}}{\partial \beta}\bigg|_{\beta=0}\right) \frac{\partial \log Z_{\beta}}{\partial \beta}\bigg|_{\beta=0} \\
    &= \left(\mathbb{D}_{\text{KL}}(\vP_1, \vP_0) + \mathbb{D}_{\text{KL}}(\vP_0, \vP_1) + 2\log\left(\frac{Z_1}{Z_0}\right)\right)\times \left(\mathbb{D}_{\text{KL}}(\vP_0, \vP_1) + \log\left(\frac{Z_1}{Z_0}\right)\right).
\end{align*}

Finally, for the last term, we have:
\begin{equation}\label{eq:finalit2}
    \E_{\beta\sim U([0,1])} \left[\left(\frac{\partial \log Z_{\beta}}{\partial \beta}\bigg|_{\beta=0}\right)^2\right] = \left(\mathbb{D}_{\text{KL}}(\vP_0, \vP_1) + \log\left(\frac{Z_1}{Z_0}\right)\right)^2.
\end{equation}

Putting everything together, we obtain the desired upper bound:
\begin{align*}
    \E_{\beta\sim U([0,1])} \left[\left(\frac{\partial \log Z_\beta}{\partial \beta}\right)^2\right] 
    \leq 3\left(\mathbb{D}_{\text{KL}}(\vP_1, \vP_0) + \mathbb{D}_{\text{KL}}(\vP_0, \vP_1) + 2\left|\log\left(\frac{Z_1}{Z_0}\right)\right|\right).
\end{align*}
\section{Experiments Details}
In this section, we will outline the specific settings in the simulations. The covariance matrices for the 10-dimensional Gaussian change detection simulation are the following matrices
\begin{align*}
\Sigma_0 = \begin{bmatrix}
1.0 & 0.6 & 0.4 & 0.2 & 0.1 & 0.05 & 0.03 & 0.02 & 0.01 & 0.01 \\
0.6 & 1.0 & 0.5 & 0.3 & 0.1 & 0.04 & 0.02 & 0.02 & 0.01 & 0.01 \\
0.4 & 0.5 & 1.0 & 0.4 & 0.3 & 0.1 & 0.05 & 0.03 & 0.02 & 0.01 \\
0.2 & 0.3 & 0.4 & 1.0 & 0.5 & 0.3 & 0.1 & 0.04 & 0.03 & 0.02 \\
0.1 & 0.1 & 0.3 & 0.5 & 1.0 & 0.6 & 0.4 & 0.2 & 0.1 & 0.05 \\
0.05 & 0.04 & 0.1 & 0.3 & 0.6 & 1.0 & 0.5 & 0.3 & 0.2 & 0.1 \\
0.03 & 0.02 & 0.05 & 0.1 & 0.4 & 0.5 & 1.0 & 0.6 & 0.4 & 0.3 \\
0.02 & 0.02 & 0.03 & 0.04 & 0.2 & 0.3 & 0.6 & 1.0 & 0.5 & 0.4 \\
0.01 & 0.01 & 0.02 & 0.03 & 0.1 & 0.2 & 0.4 & 0.5 & 1.0 & 0.6 \\
0.01 & 0.01 & 0.01 & 0.02 & 0.05 & 0.1 & 0.3 & 0.4 & 0.6 & 1.0
\end{bmatrix},
\end{align*}

\begin{align*}
\Sigma_1 = \begin{bmatrix}
1.2 & 0.7 & 0.5 & 0.3 & 0.15 & 0.1 & 0.07 & 0.05 & 0.03 & 0.02 \\
0.7 & 1.2 & 0.6 & 0.4 & 0.2 & 0.1 & 0.05 & 0.04 & 0.03 & 0.02 \\
0.5 & 0.6 & 1.2 & 0.5 & 0.4 & 0.2 & 0.1 & 0.07 & 0.05 & 0.03 \\
0.3 & 0.4 & 0.5 & 1.2 & 0.6 & 0.4 & 0.2 & 0.1 & 0.07 & 0.05 \\
0.15 & 0.2 & 0.4 & 0.6 & 1.2 & 0.7 & 0.5 & 0.3 & 0.2 & 0.1 \\
0.1 & 0.1 & 0.2 & 0.4 & 0.7 & 1.2 & 0.6 & 0.4 & 0.3 & 0.2 \\
0.07 & 0.05 & 0.1 & 0.2 & 0.5 & 0.6 & 1.2 & 0.7 & 0.5 & 0.4 \\
0.05 & 0.04 & 0.07 & 0.1 & 0.3 & 0.4 & 0.7 & 1.2 & 0.6 & 0.5 \\
0.03 & 0.03 & 0.05 & 0.07 & 0.2 & 0.3 & 0.5 & 0.6 & 1.2 & 0.7 \\
0.02 & 0.02 & 0.03 & 0.05 & 0.1 & 0.2 & 0.4 & 0.5 & 0.7 & 1.2
\end{bmatrix}.
\end{align*}

\end{document}